\documentclass{article}

\usepackage{microtype}
\usepackage{graphicx}
\usepackage{subfigure}
\usepackage{booktabs} 

\usepackage{hyperref}

\usepackage[accepted]{icml2021}

\icmltitlerunning{What fools you makes you stronger}

\usepackage{thmtools,thm-restate}
\usepackage{comment}

\usepackage{thmtools}
\usepackage{mathtools}
\usepackage{enumerate}

\usepackage{amsfonts}
\usepackage{amsthm}
 
\newcommand{\N}{\mathbb{N}}
\newcommand{\Z}{\mathbb{Z}}

\newtheorem{theorem}{Theorem}
\newtheorem{lemma}{Lemma}
\newtheorem{observation}{Observation}

\newtheorem{definition}{Definition}

\newcommand{\e}{\epsilon}

\newlength\myindent
\DeclarePairedDelimiter{\ceil}{\lceil}{\rceil}

\begin{document}

\twocolumn[
\icmltitle{
Adversarial Robustness: What fools you makes you stronger}



\icmlsetsymbol{equal}{*}

\begin{icmlauthorlist}
\icmlauthor{Grzegorz Głuch}{to}
\icmlauthor{Rüdiger Urbanke}{to}
\end{icmlauthorlist}

\icmlaffiliation{to}{School of Computer and Communication Sciences, EPFL, Switzerland}

\icmlcorrespondingauthor{Grzegorz Głuch}{grzegorz.gluch@epfl.ch}

\icmlkeywords{Machine Learning, ICML}

\vskip 0.3in
]



\printAffiliationsAndNotice{\icmlEqualContribution} 

\begin{abstract}
We prove an exponential separation for the sample complexity between the standard PAC-learning model and a version of the Equivalence-Query-learning model. We then show that this separation has interesting implications for adversarial robustness. We explore a vision of designing an adaptive defense that in the presence of an attacker computes a model that is provably robust. In particular, we show how to realize this vision in a simplified setting.

In order to do so, we introduce a notion of a strong adversary: he is not limited by the type of perturbations he can apply but when presented with a classifier can repetitively generate different adversarial examples. We explain why this notion is interesting to study and use it to prove the following. There exists an efficient adversarial-learning-like scheme such that for every strong adversary $\mathbf{A}$ it outputs a classifier that (a) cannot be strongly attacked by $\mathbf{A}$, or (b) has error at most $\e$. In both cases our scheme uses exponentially (in $\e$) fewer samples than what the PAC bound requires.

\end{abstract}

\section{Introduction}

The field of adversarial robustness \citep{intriguingprop,neuralnetseasilyfooled} revolves around three themes: designing defenses, attacking these defenses and studying the limitations of both from a theoretical perspective. In this paper we are viewing adversarial attacks as a resource. This point of view connects these three themes.


Before we describe our main contribution in detail, let us position it in a broader context.
Whereas most of the current literature considers adversarial attacks as a problem in need of a solution, we ask if they could not be seen as a resource to be exploited. As the old adage goes: "Fool me once, shame on you; fool me twice, shame on me." In this light, we envision an adaptive defense. Starting with an initial model that might only be moderately robust we ask if, in the presence of an attacker, this model can evolve to become provably robust. This idea resembles the adversarial learning approach considered in \citet{goodfellow2014explaining,madry2018towards,tramer2018ensemble,xiao2018training}. Although our current proposal does not fully realize the general vision, it has the desired property of provably improving the learner in the presence of an adversary if we restrict ourselves to a slightly simpler setting. We accomplish this by combining the following three elements.

First: Most of the literature on adversarial robustness considers adversaries whose perturbations are in some way restricted. A very common form of restriction is to bound the perturbation in the $\ell_0, \ell_2$, or $\ell_{\infty}$ norm  \citep{certifiableRobust1, certifiableRobust2}, but other perturbations (rotations, shifts, etc.) were also considered \citep{spatialpert}. A recent paper \citep{Goldwasser} argues that in order to obtain security for real-world systems we need to consider models beyond these "restricted perturbations." Unfortunately this change comes at a price: e.g., the results in \citet{Goldwasser} are based on the assumption that the learner can decide not to give an answer for some inputs (selective learning \citep{selectiveclassifiers}) and that they see the test set upfront (transductive learning). In the spirit of allowing the attacker more freedom we argue that one of the strongest adversaries one can imagine is one that, given a classifier $f$, can sample errors of $f$ according to the data distribution $\mathcal{D}$. We will call such an adversary a "strong adversary" and explain in Section~\ref{sec:discussion} why it is natural to consider them.

Second: Inspired by the main idea of cryptography to base security on computational hardness of specific problems we condition the security of our scheme on the hardness of learning. More specifically, our adversarial learning scheme computes a classifier that is robust assuming that the underlying learning task is hard in the low error regime. 

Third: We use ideas similar to those used in boosting techniques \citep{shapireboosting} to obtain an exponential separation for sample complexity of two learning models: standard PAC-learning \citep{pacvaliant} and a version of EquivalenceQuery-learning (EQ-learning) \citep{angluinEQ}. This separation is the main technical contribution of the paper. The EQ-learning model we study was considered previously in \citet{randomcounterexampleangluin} and \citet{eqboosting}. The result we took most inspiration from is \citet{eqboosting}, where the authors develop a boosting algorithm that is applied to the EQ-learning model.

Summarizing: We show an adversarial-learning-like scheme that, if strongly attacked by adversaries with limited learning power, evolves to be robust against them.

Our result has many of the properties we were looking for. Firstly, there are no explicit restrictions put on the allowed perturbations. Secondly,  the model evolves  and the risk decays exponentially (in the number of queries) until it reaches a point (due to the hardness of the underlying learning task) where the adversary no longer is capable of  attacking the model in a strong manner. So the adversary serves as a resource that helps to train the model and make it robust.



As already mentioned, our scheme does not yet fully realize the general vision of a continuously evolving learning algorithm that provably improves its robustness in the presence of an attacker. Let us explain why.
 
Firstly, in the standard execution of the adversarial learning framework, the learner presents to the adversary at every step his current best estimate. In our result we need the learner to occasionally present to the adversary a modification of his current model. In other words, we occasionally need extra input besides standard attacks. Secondly, we assume that the adversary attacks the learner in a "strong sense", as defined above.

\section{Models of learning and attack}

We start by defining the models of learning and attacks considered in this paper.

\paragraph{Notation.} For $i \in \N$ we define $[i] := [1,2,\dots,i]$. We use $\log$ to denote the $\log$ to the base $2$. For a function $f \xrightarrow{} \{-1,+1\}$ and $ A \subseteq X$ we define $f \oplus A$ as a function that is equal to $f$ at $X \setminus A$ and flips the prediction for points in $A$. For a distribution $\mathcal{D}$ on $X$ and $A \subseteq X$ we denote by $\mathcal{D}|_A$ the conditional distribution that is equal (up to scaling) to $\mathcal{D}$ on $A$ and $0$ on the complement.

Throughout the paper we consider only the realizable version of learning. That is, we assume that there is a feature space $X$, a distribution $\mathcal{D}$ on $X$, a hypothesis class $\mathcal{H}$ and a function $h \in \mathcal{H}, h : X \xrightarrow{} \{ -1, +1 \}$, that defines the ground truth on $\mathcal{D}$. The \textit{learner} knows $X$ and $\mathcal{H}$ but neither knows $h$ nor $\mathcal{D}$. The goal of the learner is to find an $f : X \xrightarrow{} \{ -1, +1 \}$ that has small risk. The risk is defined as
$$R_{\mathcal{D},h}(f) := \mathbb{P}_{x \sim \mathcal{D}}[f(x) \neq h(x)]\text{.}$$
We will consider algorithms that have access to one of the following two oracles:

\paragraph{Example Query Oracle according to $\mathcal{D}$ ($\text{EX}_{\mathcal{D}}$).} When queried, $\text{EX}_{\mathcal{D}}$ returns $(x,h(x))$, where $ x \sim \mathcal{D}$.

\paragraph{Equivalence Query Oracle according to $\mathcal{D}$ ($\text{EQ}_{\mathcal{D}}$) \citep{angluinEQ}.} For every $f : X \xrightarrow{} \{-1,+1\}$ (not necessarily from $\mathcal{H}$) the result of querying $\text{EQ}_{\mathcal{D}}(f)$ is a counterexample to $f$ distributed according to $\mathcal{D}$. More formally it is $x \sim \mathcal{D}|_{f \neq h}$. If $R_{\mathcal{D},h}(f) = 0$ then $\text{EQ}_{\mathcal{D}}(f)$ returns "YES" indicating that $f$ is equivalent to $h$. For every $k \in \N$ we write $\text{EQ}_{\mathcal{D}}(f,k)$ to denote an oracle that returns $k$ i.i.d. samples, each generated by $\text{EQ}_{\mathcal{D}}(f)$.

Next we define the two learning models for which we later show the advertised exponential separation. All our results are based on them.

\paragraph{PAC-learning.}
We say that a learning algorithm \textbf{L} PAC-learns $\mathcal{H}$ if for every $h \in \mathcal{H}$, distribution $\mathcal{D}$, and $\e,\delta \in (0,1)$, the algorithm $\textbf{L}(\e,\delta)$ asks queries to the Example Query Oracle $\text{EX}_{\mathcal{D}}$ and with probability $1- \delta$ returns a function $f \in \mathcal{H}$ such that $R_{\mathcal{D},h}(f) \leq \e$. 

The next learning model was considered before in \citet{randomcounterexampleangluin} and \citet{eqboosting}.

\paragraph{EQ-learning}
We say that a learning algorithm \textbf{L} EQ-learns $\mathcal{H}$ if for every $h \in \mathcal{H}$, distribution $\mathcal{D}$, and $\e,\delta \in (0,1)$ the algorithm $\textbf{L}(\e,\delta)$ asks queries to the Equivalence Query Oracle $\text{EQ}_{\mathcal{D}}$ and with probability $1- \delta$ returns a function $f : X \xrightarrow{} \{ -1, +1\}$ such that $R_{\mathcal{D},h}(f) \leq \e$. 

We define the adversary as an algorithm that has access to the function it wants to attack and also to the Example Query Oracle for distribution $\mathcal{D}$ and returns points from $X$. More formally:

\begin{definition}[Adversary]\label{def:adversary}
For a feature space $X$ we define an adversary $\mathbf{A}$ as an algorithm\footnote{We use {\em algorithm} here since this seems more natural. But we do not limit the attacker computationally nor are we concerned with questions of computability. Hence, {\em function} would be equally correct.} (potentially randomized) that for every function $f : X \xrightarrow{} \{-1,+1\}$ and $x \in X$ returns $\mathbf{A}(f,x) \in X$. Moreover, for a distribution $\mathcal{D}$, we denote by $\mathbf{A}(f,EX_{\mathcal{D}})$ a distribution on $X$ that is generated according to a process: sample $x \xleftarrow{} EX_{\mathcal{D}}$, return $\mathbf{A}(f,x)$. For every $f$ we say that \textbf{A} is a \textbf{strong adversary} for $f$ if: $$\textbf{A}(f,EX_{\mathcal{D}}) = EQ_{\mathcal{D}}\text{.}$$
\end{definition}

Finally, let us define the model of attack. It is an adversarial-learning-like scheme where the learning algorithm uses the adversary to learn a more robust model.

\begin{definition}[Adversarial learning game]\label{def:game}
For a distribution $\mathcal{D}$ on $X$, a learner $\mathbf{L}$, and an adversary $\mathbf{A}$ we define an adversarial learning game as follows. Learner $\mathbf{L}$ interacts with $\mathbf{A}$ in rounds, where in each round $t$ learner $\mathbf{L}$ sends a function $f_t$ to $\mathbf{A}$ and then $\mathbf{A}$ sends a point $x_t \in X$  back to $\mathbf{L}$. At every round $\mathbf{A}$ can query $EX_{\mathcal{D}}$ once and use the result when generating $x_t$. Decisions made by $\mathbf{L}$ at round $t$ can depend on the messages exchanged before round $t$. For simplicity of the statements we assume that decisions of $\mathbf{A}$ don't depend on the history. At the end of the interaction $\mathbf{L}$ declares a function $f$.
\end{definition}

\section{Main result}

We are now ready to state the main result of the paper. The proof is deferred to the appendix:

\begin{theorem}\label{thm:ip}
For every feature space $X$, for every $\epsilon \in \left(0,\frac{1}{32} \right)$, for every $d \in \N$, for every hypothesis class $\mathcal{H}$ on $X$ of VC-dimension $d$ there exists a learning algorithm $\mathbf{L}$ such that for every distribution $\mathcal{D}$, for every ground truth $h \in \mathcal{H}$, for every adversary $\mathbf{A}$ the following holds. When $\mathbf{L}$ interacts with $\mathbf{A}$ as described in Definition~\ref{def:game} then with probability $2/3$ at least one of the two properties holds:
\begin{itemize}
    \item $\mathbf{L}$, after $O(d \cdot \text{polylog}(1/\epsilon))$ rounds of interaction with $\mathbf{A}$, returns a function $f$ such that $R_{\mathcal{D},h}(f) \leq \epsilon$,
    \item there exists $t \in [O(d \cdot \text{polylog}(1/\epsilon))]$ such that at the interaction round $t$ a function $f_t$ was presented to $\mathbf{A}$ and $\mathbf{A}(f_t, EX_{\mathcal{D}}) \neq EQ_{\mathcal{D}}(f_t)$.
\end{itemize}
Moreover, throughout the interaction only $O(\text{polylog}(1/\e))$ different functions are presented to \textbf{A}.
\end{theorem}


\section{Discussion}\label{sec:discussion}


We start with a brief review of previous approaches and the lessons we can draw from those.

The standard adversarial learning setup considered in the literature is as follows. There is a classifier $f$ that an adversary \textbf{A} wants to attack.  Having some type of access to $f$, \textbf{A} generates adversarial examples to $f$. To do so, first a sample $x$ is generated according to $x \sim \mathcal{D}$. Then \textbf{A} tries to find $x'$ such that $x$ and $x'$ are semantically indistinguishable and $f$ misclassifies $x'$. A formal definition of what "semantically indistinguishable" means is difficult to furnish and it is thus the subject of considerably discussion.  The reason for this difficulty is that we expect the machine learning model to learn what "semantically indistinguishable" means from the data and not to decide it by ourselves up front. Many proxies for "semantically indistinguishable" were used in the literature. Usually they come in the form of a restriction on the set of allowed perturbations. As we mentioned in the introduction, some of the most popular restrictions include bounds on the $\ell_0, \ell_2, \ell_{\infty}$ norm or restrictions to rotations, shifts and many others.

As argued in \citet{trameroverfitting} defenses often overfit to a particular set of allowed perturbations and the resulting classifiers remain basically undefended against other attacks. This result sparked interest in defenses that make models robust against a wide variety of perturbations, even against unforeseen ones. A recent theoretical result \citep{Goldwasser} proposed a defense that make models robust against all possible perturbations. Unfortunately to achieve this the authors need to allow the learner not to give answers for some inputs and the learner sees the test set before computing the classifier. 

As shown in \citet{computationalhardness} and \citet{robustnessaccuracy} finding defenses might not be possible even in a simple case of $\ell_2$-bounded perturbations, when the adversary is all powerful. This means that limiting the capabilities of the attacker is very likely to be necessary. Different limitations on the model of attack and the power of the adversary were considered. Early papers limited the type of access that \textbf{A} has to $f$: instead of full knowledge of $f$ (known as white-box model, see \citet{evassionattackontesttime,bestwhiteboxmadrychallenge}), a black-box model \citep{blackboxsurvey}, partial white-box, where the adversary sees the logits of the output probabilities but doesn't see the internal nodes of a network, oracle access to a gradient of $f$ and others were considered. Unfortunately, even in the most restrictive model, namely the black-box model, efficient attacks have been shown to exist \citep{goodfellowblackbox,zerothblackbox,transferableblackbox,blackboxfirstmadrychallenge,adversarialserviceblackbox}. This lead the researches to explore models that limit the power of the adversary even further. In \citet{querycomplexity} the authors consider a version of the black-box, where the attacker is limited by the number of evaluations of $f$ it can perform. It is shown that classifiers with high entropy of decision boundaries are hard to attack. In \citet{berkleyguycrypto} the authors consider an adversary that is limited computationally. They show that there exist learning problems that can be attacked by an all powerful adversary but are secure against polynomially bounded attackers. These approaches however show security only for some synthetic distributions..

Summarizing this discussion we can formulate two conclusions:
\begin{enumerate}
    \item For a model to be secure in real-world applications we need to remove most of the restrictions on allowed perturbations. \label{bullet:norestr}
    \item Limiting the power of the adversary is likely inevitable. \label{bullet:limitheadv}
\end{enumerate}

We now explain how our notion of a strong adversary addresses point \ref{bullet:norestr}.

\paragraph{Why strong adversaries?}

First imagine that we removed the restriction on the perturbations that the adversary can apply completely. Defending against such attackers is impossible, because as long as \textbf{A} finds a single error $x'$ of $f$ it can map any input $x$ to $x'$. That however goes against the very intuition of what adversarial examples are. Those examples should be semantically indistinguishable from points sampled from $\mathcal{D}$. But in the above scenario the claimed "adversarial example" is always the same point! This indicates that completely removing all restrictions is likely not the model we should consider.

So how can we capture the intuition that essentially all perturbations should be allowed but avoid degenerate cases as above? First, note that an adversary as described above is easy to defend against. Simply declare $x'$ to be a "suspicious" input. As a reaction, the adversary might try to fool this defense by presenting the learner $x'$ with small random perturbations. The defense would then likely adapt by learning the type of perturbations the adversary applies and then declare a broader class of inputs as "suspicious," leading to an arms race of defense versus offense.

We break this cycle and argue that for a given classifier $f$ the strongest adversary for $f$ is the one that can generate adversarial examples exactly from the error set of $f$ according to distribution $\mathcal{D}$. More formally (as in Definition~\ref{def:adversary}) we say that \textbf{A} is a strong adversary for $f$ if by having sample access to $\mathcal{D}$ (a.k.a. access to the Example Query Oracle for $\mathcal{D}$) it can generate counterexamples from $\mathcal{D}|_{f \neq h}$, which means that it emulates the $\text{EQ}_{\mathcal{D}}$. 

This adversary has some important properties we were looking for:

There are no explicit restrictions on the type of perturbations \textbf{A} can apply. By getting $x \xleftarrow{} \text{EX}_{\mathcal{D}}$ the adversary is presented with a challenge to find an $x'$, which is semantically indistinguishable from $x$ but is not restricted to what it can do to $x$. The only restriction is that the adversarial examples he produces are statistically indistinguishable from errors of $f$ sampled according to $\mathcal{D}$.

The adversarial examples generated by \textbf{A} don't follow any particular pattern known to the learner. It is because they are distributed according to $\mathcal{D}|_{f \neq h}$ and it is natural to assume that the learner cannot distinguish them from samples from $\mathcal{D}$. Because if he did then, intuitively, he could find an estimate with a lower error as he would know where he makes mistakes. Thus, it is a priori not clear how one would defend against such an adversary.

Moreover, strong adversaries do exist. The adversary doesn't need to know $\mathcal{D}$ exactly even though he is required to sample adversarial examples from $\mathcal{D}|_{f \neq h}$. When presented with $x \xleftarrow{} \text{EX}_{\mathcal{D}}$ he knows the region of $X$ in which he should look for adversarial examples. To see that consider for instance the concentric spheres dataset from the seminal work \citet{adversarialSpheres}. In there it was shown that the error sets of classifiers can be understood as spherical caps. Then it was argued that in high dimensions these spherical caps, even though having very small probability, are very close (in the $\ell_2$ sense) to a constant fraction of the distribution. A typical $\ell_2$ bounded attack finds adversarial examples that are, due to symmetry, distributed approximately uniformly on this spherical cap. This means that the adversary approximately satisfies our assumption of being strong for this particular classifier. But the existence of strong adversaries doesn't only happen for synthetic distributions. The very observation that it is hard to design defenses against adversaries in practice points to the fact that these adversaries are strong for the type of classifiers that are learnt. After all, being unable to defend against them intuitively means that the adversarial examples generated by these attacks don't contain any structure or property we can detect apart from belonging to the error set of $f$. This translates to these examples being distributed approximately according to $\mathcal{D}|_{f \neq h}$.

\paragraph{Implications of our main result}
Now we explain how one can understand our main result (Theorem~\ref{thm:ip}).

The first point of view is that Theorem~\ref{thm:ip} explores implications on the existence of strong adversaries. Interpreted as such, it shows that there exists an adversarial-learning-like scheme where strong adversaries can be used to learn a hypothesis of error $\e$ exponentially faster than guaranteed by the standard learning theory results.

The second, and arguably the more interesting point of view is the following.
If we assume that the adversary we want to defend against is unable to learn a classifier with error $\e$ using $O(d \cdot \text{polylog}(1/\e))$ samples then the adversarial-learning scheme from Theorem~\ref{thm:ip} can be understood as a defense. It is because throughout 
the execution of the protocol there was a function $f_t$ that was presented to \textbf{A} such that $\mathbf{A}(f_t, EX_{\mathcal{D}}) \neq EQ_{\mathcal{D}}(f_t)$. Thus the protocol generates a list of $O(\text{polylog}(1/\e))$ many functions such that \textbf{A} is not a strong adversary for at least one of them. The fact that $\mathbf{A}(f_t, EX_{\mathcal{D}}) \neq EQ_{\mathcal{D}}(f_t)$ implies that there is a structure in the adversarial examples generated by \textbf{A} other than just being in the error set of $f_t$. The structure can be of a form of concentrating adversarial examples in particular regions of the feature space (as in the example we considered at the beginning of the \textbf{Why strong adversaries?} paragraph) or concentrating on points with a particular property. 

The fact that this structure exists opens a door for designing defenses that are provably secure against adversaries that are limited by what they can learn. Assume we introduced a distance $\text{dist}$ on distributions, think of the KL divergence or the earth mover distance. Then we can imagine that an extension of Theorem~\ref{thm:ip} can guarantee that when \textbf{A}'s learning power is limited then the scheme computes $f_t$ such that $\text{dist}(\mathbf{A}(f_t, EX_{\mathcal{D}}), EQ_{\mathcal{D}}(f_t)) \geq \eta$ for some parameter $\eta$. If $\eta$ is big enough one can hope to design a distinguisher between the two distributions and thus detect adversarial examples.

\paragraph{Why it makes sense to assume the learning problem is hard?}
As argued in~\ref{bullet:limitheadv} above, limiting the power of the adversary is most likely unavoidable. 

Imagine that the adversary you want to protect your model against is a very good learner. More formally, assume that for a distribution $\mathcal{D}$ and a ground truth $h \in \mathcal{H}$ the adversary can compute a classifier $g$ of very small risk $R_{\mathcal{D},h}(g) \approx 0$. Then when the adversary attacks $f$ it can compute $[f \neq g] \subseteq X$. Note that as the risk of $g$ is close to $0$ we know that $\mathbb{P}_{x \sim \mathcal{D}}[f(x) \neq g(x)] \approx R_{\mathcal{D},h}(f)$, that is the region $[f \neq g]$ contains almost all errors of $f$. Even though \textbf{A} might not know the data distribution $\mathcal{D}$ (we only assumed that \textbf{A} is able to find a classifier with low error, which doesn't necessarily imply the knowledge of the distribution) it is still able to attack $f$. To do that, for test example $x \sim \mathcal{D}$, \textbf{A} can find $x' \in [f \neq g]$ that is semantically closest to $x$. Then if the error set of $f$ is semantically close to all of the data distribution then $f$ is indefensible against all powerful learners. 

Our result provides an exponential separation between sample complexities, as the standard PAC-bound requires $O\left( \frac{d}{\e} \right)$ samples. This resembles the types of separation results, which cryptography is built on and gives even more hope for following this line of research to find provably secure defenses. Unfortunately this separation doesn't happen for all distributions as the PAC guarantee is only an upper bound. We also know that VC-theory does not necessarily give tight bounds for distributions encountered in practice. We note however that our result if understood as a boosting technique can be applied to any learning algorithm and thus the implications are not necessarily restricted to the VC-theory (for a discussion about that see Section~\ref{sec:interactiveproofs}). Even if not perfect we see our result as an important starting point in investigating the interplay of learnability and adversarial robustness.


\section{The Exponential Separation}

In this section we present the main technical result of the paper, which is an exponential separation for the sample complexity between PAC-learning and EQ-learning. At the end of this section we show how Theorem~\ref{thm:ip} follows.

A slightly improved version of the PAC bound in the realizable case \citep{optimalpac} states that for a hypothesis class $\mathcal{H}$ of VC-dimension $d$ in order to learn (with constant probability) a classifier of risk $\e$ it suffices to use $O\left(\frac{d}{\e} \right)$ samples. This result is tight in a sense that there exist hypothesis classes and distributions for which that many samples are necessary. Our main technical result guarantees that in the EQ-learning model $O(d \cdot \text{polylog}(1/\epsilon))$ many queries suffice. This is an exponential improvement. Why is this possible?

Imagine that there is a ground truth $g \in \mathcal{H}$ that you try to learn and that you already found $h \in \mathcal{H}$ such that $R_{\mathcal{D},g}(h) \leq \eta$. Then the counterexample oracle $\text{EQ}_{\mathcal{D}}(h,\cdot)$ provides you with samples from a distribution $\mathcal{D}|_{h \neq g}$. Querying the oracle $O(d)$ times you get a sample $S \xleftarrow{} \text{EQ}_{\mathcal{D}}(h,O(d))$. You can now use the PAC bound: if you find an $h' \in \mathcal{H}$ that is consistent with $S$ then you know that it has an error of at most $1/2$ on $\mathcal{D}|_{h \neq g}$. It is then natural to define
\[ \text{Combine}(h,h')(x) = \begin{cases} h(x), & \mbox{if } x \in h = g \\ h'(x), & \mbox{if } x \in h \neq g \end{cases} \text{.}\]
Note that $\text{Combine}(h,h')$ has an error of at most $\eta/2$. If we repeated this procedure $\log(1/\e)$ times, thus asking $O(d \log(1/\e))$ queries, you would find a classifier with error $\e$. 

Unfortunately it is not possilbe to compute $\text{Combine}(h,h')$. Afterall, if you knew the region of the input space where $h \neq g$ then you could just flip the prediction of $h$ in that region and the resulting classifier would have zero error. But it turns out that the general intuition of decreasing the error by a multiplicative factor after every $O(d)$ queries can indeed be achieved. The key to this result is to find a computable version of $\text{Combine}(h,h',...)$ that guarantees an exponentially fast decay of the error.

Our algorithm is mainly inspired by boosting techniques, most notably by an approach from \citet{eqboosting}. In this work the authors consider capabilities of polynomially bounded learners. 
If we assume that the "complexity" of the hypothesis class $\mathcal{H}$ is measured by its VC-dimension then
the relevant result from \citet{eqboosting} can be summarized as follows. If we have an algorithm $\textbf{L}^{\text{PAC}}$ that learns $\mathcal{H}$ to a constant error in time (the authors focus on time but the time is of course an upper-bound for the sample complexity) $\text{time}(d)$ then this algorithm can be boosted to an algorithm $\textbf{L}^{\text{EQ}}$ that learns $\mathcal{H}$ in the EQ-model to an error $\frac{1}{\omega(\text{poly}(d))}$ in time $\text{poly}(\text{time}(d))$. 

Let us apply this boosting technique to our setting. Assume that $\textbf{L}^{\text{PAC}}$ runs in time $O(d)$ (as this is the number of samples that are required by the standard PAC bound to learn to constant error). Then the boosting algorithm can be used to produce $\textbf{L}^{\text{EQ}}$ that learns $\mathcal{H}$ to error $\e$. What bound on the run time of $\textbf{L}^{\text{EQ}}$ do we get?  Unfortunately this bound is no better than $O \left(\frac{d}{\e} \right)$. This is exactly what the standard PAC bound provides in the first place. Thus, disappointingly, a direct application of these ideas do not yield a benefit in using the EQ-model versus using the PAC-model. 

To get the claimed exponential separation between the two models we develop a boosting-like algorithm that differs significantly in several important aspects and hence also requires a different proof technique. The main idea is to compute $h_1, \dots, h_{O(\text{polylog}(1/\e))} \in \mathcal{H}$ in a sequential manner and then to define the final hypothesis as a version of a majority vote of these functions. 

The simplified, high level, structure of the algorithm is as follows. Repeat the following for $t = O(\text{polylog}(1/\e))$ steps: at step $i$ ask the oracle for $S_i \xleftarrow{} \text{EQ}_{\mathcal{D}}(\text{"Majority"}(h_1,\dots,h_{i-1}), O(d))$ and then define $h_i := \text{FindConsistent}(S_i,\mathcal{H})$ (FindConsistent returns a function from $\mathcal{H}$ that agrees with all samples from $S_i$). At the end return $\text{"Majority"}(h_1, \dots, h_{t})$. To make this approach work we need to ensure that the error sets of $h_1, \dots, h_{t}$ are sufficiently independent. 

The following points need particular attention. First, note that the $\text{EQ}_{\mathcal{D}}$ provides the algorithm with samples from the error set only. Thus, correctly classified points at one stage will not automatically remain correctly classified at later stages. To make sure this happens, at every step $i$ we include in the training set $S_i$ samples from carefully chosen regions of the feature space that are already classified correctly (this is done in the inner "for" loop of the algorithm). Second, a simple majority vote of the previously constructed classifiers is not sufficient to get the desired result. This is true since a non-negligible region of the feature space might become incorrectly classified with higher and higher confidence by such a majority vote. This is the reason we clip the values of votes to a fixed interval (for details see Definition~\ref{def:majority}).

We start now with the formal definition of the algorithm, followed by its proof.

\subsection{The Algorithm}

First we give formal definitions of the concepts used in the algorithm. All proofs are deferred to the appendix.

\begin{algorithm}[tb]
   \caption{EQlearner}
   \label{alg:learner}
\begin{algorithmic}
   \STATE {\bfseries Input:} hypothesis class $\mathcal{H}$ of VC-dimension $d$, target error $\epsilon$, target confidence $\delta$, Equivalence query oracle $EQ_{\mathcal{D}}$.
   \STATE
   \STATE 
   $\e' := \frac{\e}{10^5 \log^4(1/\e)}$
   \STATE $B_{\e'} := 2\ceil{\log(1/\e')} + 1$
   \STATE $m := O \left((d + \log(B_{\e'}^4) + \log(1/\delta)) \cdot B_{\e'}^4 \right)$
   \STATE $t := O \left(B_{\e'}^3 \right)$
   \STATE $h_1 \in \mathcal{H}$
   \FOR{$i=2$ {\bfseries to} $t$}
        \STATE $S_i := 
        \text{EQ}_{\mathcal{D}}(\text{Maj}(h_1,\dots,h_{i-1}),m)$
        \FOR{$v \in [B_{\e'}] \cap 2\Z + 1$}
            \STATE $h' := \text{Maj}(h_1,\dots,h_{i-1}) \oplus [\text{Vote}(h_1,\dots,h_{i-1}) \in \{v,-v\}]$
            \STATE $S_i^v := \text{EQ}_{\mathcal{D}}(h',m)$
        \ENDFOR
        \STATE $h_i := \text{FindConsistent} \left(S_i \cup \bigcup_{v \in [B_{\e'}]} S_i^v, \mathcal{H} \right)$
    \ENDFOR
   \STATE
   \STATE {\bfseries Return} $\text{Maj}(h_1,h_2,\dots,h_{t})$
\end{algorithmic}
\end{algorithm}

\begin{definition}[Vote and Majority]\label{def:majority}
For $\e \in (0,1)$ we define $B_\e := 2\ceil{\log(1/\e)}+1$ and $\text{clip}_{\e} : \Z \xrightarrow{} \Z$ as: 
$$\text{clip}_{\e}(x) : = \min(\max(-B_\e,x), B_\e)\text{.} $$
For a sequence of functions $h_1,\dots,h_i : X \xrightarrow{} \{-1,+1\}$, $\e \in (0,1)$ and $x \in X$, we define $\text{Vote}(h_1,\dots,h_i)(x)$ recursively as: 
\begin{align*}
&\text{Vote}(h_1,\dots,h_i)(x) := \\ &\text{clip}_{\e} \left(\text{Vote}(h_1,\dots,h_{i-1})(x) + 2h_i(x) \right) \text{,}\\
&\text{Vote}(h_1)(x) := h_1(x) \text{.}    
\end{align*} 
Similarly, for a sequence of functions $h_1,\dots,h_i : X \xrightarrow{} \{-1,+1\}$, $\e \in (0,1)$ and a ground truth function $g$ we define $\text{Vote}_g(h_1,\dots,h_i)(x)$ recursively as:
\begin{align*}
&\text{Vote}_g(h_1,\dots,h_i)(x) := \\ &\text{clip}_{\e} \left(\text{Vote}(h_1,\dots,h_{i-1})(x) + 2 \cdot (-1)^{h_i(x) = g(x)} \right) \text{,}\\
&\text{Vote}_g(h_1)(x) := (-1)^{h_1(x) = g(x)} \text{.}    
\end{align*} 

\paragraph{Note.} $\text{Vote}(h_1,\dots,h_i)(x)$ expresses our current estimate for a particular input (together with a level of confidence), whereas $\text{Vote}_g(h_1,\dots,h_i)(x)$ denotes the error of this current estimate with respect to the ground truth.

Finally, we define: 
\[ \text{Maj}(h_1, \dots, h_i)(x) := \begin{cases} +1, & \mbox{if } \text{Vote}(h_1,\dots,h_i)(x) \geq 0, \\ -1, & \mbox{otherwise.} 
\end{cases}\]
\end{definition}

\begin{observation}
Observe that for all $i \in \N$, $h_1, \dots, h_i,g : X \xrightarrow{} \{-1,+1\}$ and $x \in X$ we have:
$$\text{Vote}(h_1,\dots,h_i)(x) \in 2\Z+1 \cap [-B_\e , B_\e] \text{ and}$$
$$\text{Vote}(h_1,\dots,h_i)(x) = \pm \text{Vote}_g(h_1,\dots,h_i)(x) \text{.}$$
\end{observation}

\subsection{Proof of correctness}


\begin{restatable}{lemma}{lempositivegoesdown}
\label{lem:34ofpositivegoesdown}
For every $i \in \N$, for every $h_1, \dots, h_{i-1} \in \mathcal{H}$, for every ground truth $g \in \mathcal{H}$ and for every $\delta \in (0,1)$: if $m = \Omega(d + \log(1/\delta))$ then with probability $1-\delta$ every function $h \in \mathcal{H}$ that is consistent with $S \sim \text{EQ}_{\mathcal{D}}(\text{Maj}(h_1,\dots,h_{i-1}),m)$ satisfies the following:
\begin{align*}
&\mathbb{P}_{x \sim \mathcal{D}}[\text{Maj}(h_1,\dots,h_{i-1})(x) \neq g(x) \wedge h(x) \neq g(x) ] \\
&\leq \frac{1}{16}\mathbb{P}_{x \sim \mathcal{D}}[\text{Maj}(h_1,\dots,h_{i-1})(x) \neq g(x)].
\end{align*}
\end{restatable}

In words, the lemma states that if we get $m$ samples and take any function in the hypothesis class that is consistent with those samples, this function will be incorrect at most on a fraction $1/16$ of the error set of our current estimate. Note that the required "independence" of new functions in our boosting-like algorithm is partially satisfied by this statement.


\begin{restatable}{lemma}{lemifbigthanthreshold}
\label{lem:ifbiggerthanthreshgoesdown}
For every $\e' \in (0,1)$, $i \in [t]$, $h_1, \dots, h_{i-1} : X \xrightarrow{} \{ -1,+1\}$, ground truth $g \in \mathcal{H}$ and $v \in [B_{\e'}] \cap 2\Z+1$ if: 
\begin{align*}
&\mathbb{P}_{x \sim \mathcal{D}}[\text{Vote}_g(h_1,\dots,h_{i-1})(x) = -v] \\
&\geq \frac{1}{B_{\e'}^4} \mathbb{P}_{x \sim \mathcal{D}}[\text{Maj}(h_1,\dots,h_{i-1})(x) \neq g(x)]
\end{align*}
then for $m = \Omega((d+\log(1/\delta))B_{\e'}^4)$ we have that with probability $1-\delta$ every function $h \in \mathcal{H}$ that is consistent with $\text{EQ}_{\mathcal{D}}(h',m)$, where $h' := \text{Maj}(h_1,\dots,h_{i-1}) \oplus [\text{Vote}(h_1,\dots,h_{i-1}) \in\{v,-v\}]$ satisfies the following:
\begin{align*}
&\mathbb{P}_{x \sim \mathcal{D}}[\text{Vote}_g(h_1,\dots,h_{i-1})(x) = -v  \wedge h(x) \neq g(x) ] \\
&\leq \frac{1}{16}\mathbb{P}_{x \sim \mathcal{D}}[\text{Vote}_g(h_1,\dots,h_{i-1})(x) = -v]
\end{align*}
\end{restatable}
Recall that our classifier is based on a sequence of classifiers. Each of these classifiers casts a vote. Those votes are tallied and possibly clipped. The final classifier looks at the sign of the vote count. We can think of the vote count as the "confidence" we have in the particular decision.  Consider all the points in the feature space that have a particular vote count. Assume that this vote count is negative (correct decision) and that this particular vote count has a large probability mass. The lemma then states the following. If we get $m$ further samples and take any function in the hypothesis class that is consistent with those samples, then this function will be incorrect at most on a fraction $1/16$ of the points with this particular vote count.

Next we define an abstract process on odd integers. This process will emulate how a collection of the following probabilities evolves throughout the execution of the algorithm. For iteration $t$ of the algorithm, and a vote value $i \in 2\Z + 1$ we think that $p_i^t$ (which is defined below) is equal to $\mathbb{P}_{x \sim \mathcal{D}}[\text{Vote}_g(h_1,\dots,h_{i-1})(x) = i]$. The two properties defined in Definition~\ref{def:process} correspond to Lemma~\ref{lem:34ofpositivegoesdown} and Lemma~\ref{lem:ifbiggerthanthreshgoesdown}. We refer the reader to Figure~\ref{fig:voting} for a visual representation of the rules of the process. The values $\{p_i\}_{i \in I_\e}$ are arranged on a line, each $p_i$ corresponds to one rectangle. The horizontal dashed line represents the threshold at which the second property from Definition~\ref{def:process} is triggered. The left/right arrows and the values next to them represent how much mass is moved to the left and to the right from a given position.

\begin{figure*}
  \centering
  \includegraphics[width=1\textwidth]{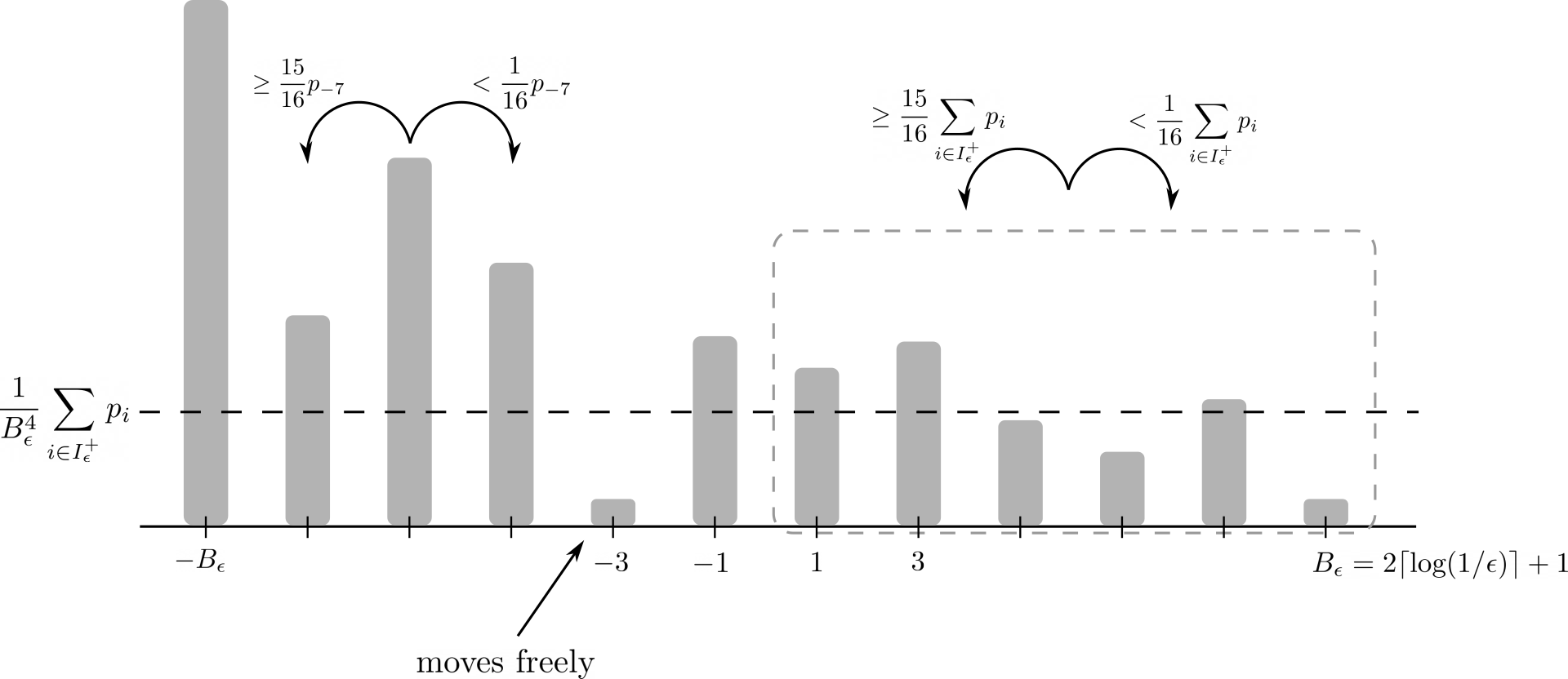}
  \caption{Visualization of the process.}
  \label{fig:voting}
\end{figure*}

\begin{definition}[Process on $2\Z+1$]\label{def:process}
For every $\e \in (0,1)$ we define a process on $I_\e := 2\Z+1 \cap [-B_\e, B_\e]$. For simplicity we introduce the notation $I_{\e}^+ := I_\e \cap (\Z > 0), I_{\e}^- := I_\e \cap (\Z < 0)$. For every $i \in I_\e$ and $t \in \N$, there is a value $p_i^t$ associated with a point $i$ at time step $t$. The process starts from an initial configuration  $\{ p_i^1\}_{i \in I_{\e}}$, such that $\sum_{i \in I_{\e}} p_i^1 = 1$. For step $t \in \N$ and for every $i \in I_{\e}$ the weight $p_i^t$ is split into two parts: a part of $p_i^t$ moves to $i-2$ and the remaining part moves to $i+2$. More precisely, this is done in the following manner:
\begin{itemize}
    \item At every step $t$ at least $\frac{15}{16} \sum_{i \in I_\e^+} p_i^t$ of the mass, i.e., at least $\frac{15}{16}$ of the mass on $I_{\e}^+$, moves down.
    
    \item At every step $t$ and for every $i \in I_\e^+$, if $$p_i^t \geq \frac{1}{B_\e^4} \sum_{i \in I_\e^+} p_i^t$$
    then at most $\frac{1}{16} p_i^t$ of the weight from $p_i^t$ moves to $i+2$.
\end{itemize}
If some mass moved to $-B_\e-2$ or $B_\e+2$ then it is moved back to $-B_\e$ and $B_\e$, respectively.
\end{definition}


According to Definition~\ref{def:process}, as long as there is any "substantial" mass on a position $i < 0$, at least $15/16$ of this mass has to move two positions down and at most $1/16$ can move two positions up. Moreover $15/16$ of the mass on $i > 0$ has to move down. It is therefore intuitively not surprising that we expect less and less mass to be found on the positive part and the process continues. Lemma~\ref{lem:processconverges} makes this intuition quantitative.

\begin{restatable}{lemma}{lemprocessconverges}
\label{lem:processconverges}
Let $\e \in \left(0,\frac{1}{32} \right)$ and consider an initial configuration $\{p_i^1\}_{i \in I_\e}$ such that $\sum_{i \in I_\e} p_i^1 = 1$. Then after $t = O(B_\e^3)$ steps of the process
$$\sum_{i \in I_{\e}^+ } p_i^t \leq 64 \cdot \e \cdot B_\e^3 \text{.}$$
\end{restatable}

To get the final result it is enough to to take the union bound over the failure events of Lemma~\ref{lem:34ofpositivegoesdown} and~\ref{lem:ifbiggerthanthreshgoesdown} and then apply Lemma~\ref{lem:processconverges}.

\begin{restatable}{theorem}{thmqc}
\label{thm:querycomplexity}
For every $\e \in \left(0,\frac{1}{32}\right), \delta \in (0,1)$, every hypothesis class $\mathcal{H}$ of VC-dimension $d$, for every distribution $\mathcal{D}$ we have that EQlearner (Algorithm~\ref{alg:learner}) run with parameters $\e, \delta, \mathcal{H}$ EQ-learns $\mathcal{H}$ asking 
$$O((d + \log(1/\delta))\log^9(1/\e)) \text{ queries.}$$
\end{restatable}

\paragraph{Note.} Optimizing the power on $\log(1/\e)$ in the query upper-bound was not our priority. We focused on simplicity of the algorithm and clarity of the proof. We believe that one can improve the analysis to get a tighter bound. We also think that one would have to come up with a new algorithm to prove that the query complexity of the EQ-model belongs to $o(d \cdot \log^2 (1/\e))$.

Our main result, Theorem~\ref{thm:ip} is an easy consequence of Theorem~\ref{thm:querycomplexity}.

\section{EQ-learner as a Booster}\label{sec:interactiveproofs}

Guarantees based on the VC-dimension are often not tight and as our result is phrased in these terms one might wonder how much it depends on this specific measure. As mentioned, our algorithm can be understood as a boosting technique and hence the result applies more generally. Next, we explain what we mean by that.


Let $\mathcal{H}$ be a hypothesis class and $\mathcal{D}$ be a distribution. Imagine that we have an algorithm $\mathcal{A}$ that for some distributions learns $\mathcal{H}$. What we mean is that $\mathcal{A}$ is a "PAC-learner" for $\mathcal{H}$ but only for a class of distributions $\mathfrak{D}$. Then imagine that we use $\mathcal{A}$ as a subroutine in the EQ-learning algorithm. I.e., instead of following the template of $S := \text{EQ}_{\mathcal{D}}(f,m), h := \text{FindConsistent}(S, \mathcal{H})$ (as in Algorithm~\ref{alg:learner}) we use $\mathcal{A}$ to get an $h$ that has a small error on distribution $\text{EQ}_{\mathcal{D}}(f)$. Now assume that all distributions $\text{EQ}_{\mathcal{D}}(f)$ for which $\mathcal{A}$ is run belong to $\mathfrak{D}$.

In this case a slight extension of Theorem~\ref{thm:querycomplexity} shows that we can boost $\mathcal{A}$ in the following sense. Assume that for every $\e$, $\delta \in (0,1)$, every $\mathcal{D}' \in \mathfrak{D}$ $\mathcal{A}$ learns $\mathcal{H}$ on $\mathcal{D}'$ in $Q_{\mathcal{A}}(\mathcal{H},\e,\delta)$ number of samples. Then there exists an EQ-learner (this is Algorithm~\ref{alg:learner}, which uses $\mathcal{A}$ as a subroutine) that learns $\mathcal{H}$ up to error $\e$ in number of queries upper-bounded by 
$$Q_{\mathcal{A}} \left(\mathcal{H},\frac{1}{16},\frac{\delta}{\text{polylog}(1/\e)} \right) \cdot \text{polylog}(1/\e) \text{.}$$ 
Now observe that if
\begin{equation}\label{eq:boostingassumption}
\frac{Q_{\mathcal{A}} \left(\mathcal{H},\e,\delta \right)}{Q_{\mathcal{A}} \left(\mathcal{H},\frac{1}{16},\frac{\delta}{\text{polylog}(1/\e)} \right)} \gg \text{polylog}(1/\e) \text{,}
\end{equation}
then the constructed EQ-learner learns $\mathcal{H}$ to error $\e$ with fewer queries than $\mathcal{A}$ does in the PAC-model. This is a different type of separation result. In words it says that in some cases you can boost an algorithm from the PAC-model to the EQ-model such that fewer queries are required. The condition from \eqref{eq:boostingassumption} in words means that the dependence of the runtime of $\mathcal{A}$ on $\e$ grows faster than $\text{polylog}(1/\e)$. This is a reasonable assumption as in the PAC-learning model for every hypothesis one needs $\Omega(1/\e)$ samples just to see a single point from the error set of this hypothesis. This suggests that the dependence of $Q_{\mathcal{A}} \left(\mathcal{H},\e,\delta \right)$ on $\e$ might grow like $\Omega(1/\e)$ (which is exactly what happens in the standard PAC-bound). To summarize, even in the cases when the VC-theory is far from reality one can still hope to get interesting results using our technique.

\section{Conclusions and Open Problems}

We study the interplay between attacks, defenses and learnability in the context of adversarial robustness. 

We start from the main lessons learned from past experimental and theoretical work on this topic, namely that for models to be secure in real-world applications we need to remove most of the restrictions on allowed perturbations and that some limitation on the power of the adversary is likely necessary to achieve any kind of guarantees. 

We ask whether, rather than fighting the adversary, one can use his power to evolve a given learning scheme to become increasingly robust. We  then introduce a learning setting where such a program can indeed be carried out. 

The core technical contribution on which our result is based is an exponential separation between the PAC-learning and the EQ-learning.

Even though our result still falls short of providing a provable defense in real-world settings, it has many of the properties that we believe such a system ought to have and we hope that it provides a blue print of how such a goal might be achieved. One possible recipe to achieve this goal could be as follows: 
\begin{enumerate}
    \item Show a separation between the PAC and EQ learning for the class of learning problems you are interested in.
    \item Our result guarantees that adversarial examples are distinguishable from mere errors.
   \item Design a distinguisher for a pair of distributions to detect adversarial examples.
\end{enumerate} 



Apart from the ambitious goals mentioned above our work poses also some theoretical open problems. The query complexity upper-bound in the EQ-model we were able to prove is of the form $O \left(d \cdot \log^9 \left(1/\e \right) \right)$. As we mentioned before this is not optimal and what we believe to be the true query complexity is $O(d \cdot \log(1/\e))$. Proving an upper or a lower-bound close to this expression is an interesting theoretical challenge.

\bibliography{example_paper}
\bibliographystyle{icml2021}


\newpage

\appendix
\section{Proofs}

We start by recalling the standard PAC upper-bound for the sample complexity of learning in the realizable case. 

\begin{lemma}\label{lem:standardPAC}
For every hypothesis class $\mathcal{H}$ of VC-dimension $d$ we have that for every $\e,\delta \in (0,1)$ $\mathcal{H}$ is PAC-learnable using the FindConsistent algorithm with sample complexity:
$$\frac{d \log(1/\e) + \log(1/\delta)}{\e}\text{.} $$
\end{lemma}

\lempositivegoesdown*

\begin{proof}
Let $i \in \N$, $h_1, \dots, h_{i-1} : X \xrightarrow{} \{-1,+1\}$. Note that $\text{EQ}_{\mathcal{D}}(\text{Maj}(h_1,\dots,h_{i-1}),m)$ generates $m$ i.i.d. samples from the distribution $\mathcal{D}|_{ \text{Maj}(h_1, \dots, h_{i-1}) \neq g}$. Then Lemma~\ref{lem:standardPAC} guarantees that if $m = \Omega(d + \log(1/\delta))$ then with probability $1-\delta$ every $h \in \mathcal{H}$ that is consistent with $m$ i.i.d. samples from $\mathcal{D}|_{ \text{Maj}(h_1, \dots, h_{i-1}) \neq g}$ has error at most $\frac{1}{16}$ on $\mathcal{D}|_{ \text{Maj}(h_1, \dots, h_{i-1}) \neq g}$. This is equivalent to the statement of the Lemma.
\end{proof}

\lemifbigthanthreshold*

\begin{proof}
Let $i \in [t], v \in [B_{\e'}] \cap 2\Z+1$ and $S \sim \text{EQ}_{\mathcal{D}}(h',m)$. We will show that with high probability the following holds:
$$|\{x \in S : \text{Vote}_g(h_1,\dots,h_{i-1})(x) = -v \}| \geq \frac{m}{8 B_{\e'}^4} \text{.}$$
Let $X_i$ be a Bernoulli random variable that is equal to $1$ if and only if the $i$-th sample from $S$ belongs to the region $\text{Vote}_g(h_1,\dots,h_{i-1})(x) = -v$. These random variables are independent and each has success probability $p$, which we claim is at least: 
$$\frac{\mathbb{P}_{\mathcal{D}}[\text{Vote}_g(h_1,\dots,h_{i-1})(x) = -v] }{ R(\text{Maj}(h_1, \dots, h_{i-1})) + \mathbb{P}_{\mathcal{D}}[\text{Vote}_g(h_1,\dots,h_{i-1})(x) = -v] } \text{.}$$
To see that this is true note that for every $x \in X$ such that $\text{Vote}_g(h_1,\dots,h_{i-1})(x) = -v$ we have by definition of $\text{Vote}_g$ and $\text{Maj}$ that $\text{Maj}(h_1, \dots, h_{i-1}) = g(x)$. Recall that $h' = \text{Maj}(h_1,\dots,h_{i-1}) \oplus [\text{Vote}(h_1,\dots,h_{i-1}) \in\{v,-v\}]$, which means that for every $x$ such that $\text{Vote}_g(h_1,\dots,h_{i-1})(x) = -v$ we have also that $h'(x) \neq g(x)$, which means that $x$ is misclassified by $h'$. By assumption we have then that 
$$p \geq \frac{1}{\left(1 + \frac{1}{B_{\e'}^4} \right) B_{\e'}^4} \geq \frac{1}{2 B_{\e'}^4} \text{.}$$ 
We introduce the notation $a \approx_{\beta,\alpha} b$ to denote $a \in [(1-\beta)b - \alpha, (1+\beta)b + \alpha]$. By the Chernoff-Hoeffding bound we get that there exists a universal constant $\Gamma$ such that for all $0 < \beta \leq \frac{1}{2}, 0 < \alpha$:
$$
\frac{\sum_{i=1}^m X_i}{m} \approx_{\beta,\alpha} p \text{ with probability } 1 - 2e^{-\Gamma k \alpha \beta} \text{.}
$$
Setting $\beta := \frac{1}{2}, \alpha := \frac{1}{8 B_{\e'}^4}$ we get that:
$$\sum_{i=1}^m X_i \geq \frac{m}{8 B_{\e'}^4} \text{ with probability } 1 - 2e^{-  \frac{\Gamma m}{16 B_{\e'}^4} } \text{.}$$
Now observe that conditioned on a sample $x \in S$ being such that $\text{Vote}_g(h_1,\dots,h_{i-1})(x) = -v$ we know that $x$ is distributed according to $\mathcal{D}|_{\text{Vote}_g(h_1,\dots,h_{i-1})(x) = -v}$. Thus Lemma~\ref{lem:standardPAC} guarantees that if $\sum_{i=1}^m X_i \geq O(d + \log(1/\delta)) $ then with probability $1- \delta$ any function consistent with $S$ has error at most $\frac{1}{16}$ on $\mathcal{D}|_{\text{Vote}_g(h_1,\dots,h_{i-1})(x) = -v}$. So if $m = \Omega((d+\log(1/\delta))B_{\e'}^4)$ then by the union bound over the two failure events we get the result.
\end{proof}


\lemprocessconverges*

\begin{proof}
For $t\in \N$, let $\{p_i^t\}_{i \in I_\e}$ be the configuration resulting from running the process for $t$ steps. We define two metrics that will measure the progress of the process:
$$W_t := \sum_{i \in I_\e^-} 2^i \cdot p_i^t + \sum_{i \in I_{\e}^+ } p_i^t \text{, and} $$
$$M_t := \frac{\sum_{i\in I_{\e}^+ } i \cdot p_i^t}{\sum_{i\in I_{\e}^+ } p_i^t} \text{.}$$
The first metric $W_t$ is a weighted average of these masses, where more weight is put on positions that are  "to the left." The second metric $M_t$ is just the expected value of the position of all the weight on the positive part.

Let $t \in \N$. We analyze the evolution from $\{p_i^t\}_{i \in I_\e}$ to $\{p_i^{t+1}\}_{i \in I_\e}$. First note that by definition:
\begin{equation}\label{eq:wtuprbnd}
\sum_{i \in I_{\e}^+ } p_i^t \leq W_t \leq \sum_{i\in I_\e} p_i^t \leq 1 \text{.}
\end{equation}
Observe that $W_t$ is linear in the $\{p_i^t \}$'s. This means that we can analyze the contribution of each weight separately. Let us therefore analyze how the contribution from $p_i^t$ to $W$ changes as $t$ increases to $t+1$. For every $t \in \N$, any index $i \in I_\e$ belongs to one of the following types:
\paragraph{Type 1.} $i \in I_\e^- \setminus \{-B_\e\}, p_i^t \geq \frac{1}{B_\e^4} \sum_{j \in I_{\e}^+ } p_j^t$: By definition the contribution of $p_i^t$ to $W_t$ is equal to $2^i \cdot p_i^t$. At step $t+1$, $p$ of the mass moves to $i-2$ and $p_i^t - p$ moves to $i+2$. By the rules of the process (Definition~\ref{def:process}) $p \geq \frac{15}{16}p_i^t$. Hence the contribution of this mass to $W_{t+1}$ is at most
\begin{align}
&2^{i-2} \cdot p + \min(2^{i+2},1) \cdot (p_i^t - p) \nonumber  \\
&\leq  2^{i-2} \cdot \frac{15}{16} p_i^t + 2^{i+2} \cdot \frac{1}{16} p_i^t \nonumber \\
&\leq 2^{i-2} \cdot p_i^t \cdot \left(\frac{15}{16} + 1\right) \nonumber \\
&\leq \frac{1}{2} \cdot 2^i \cdot p_i^t \text{.} \label{eq:normalptsbnd}
\end{align}
This means that the \textbf{contribution to $W$ decreases by a multiplicative factor of at least $2$}.

\paragraph{Type 2.} $i \in I_\e^- \setminus \{-B_\e\}, p_i^t < \frac{1}{B_\e^4 } \sum_{j \in I_{\e}^+ } p_j^t$: As in the previous case the contribution of $p_i^t$ to $W_t$ is equal to $2^i \cdot p_i^t$. By the rules of the process $0 \leq p \leq p_i^t$ of the mass moves to $i-2$ and $p_i^t - p$ moves to $i+2$. Thus the contribution of this mass to $W_{t+1}$ is at most:
\begin{equation}\label{eq:smallpointsuprbnd}
2^{i-2} \cdot p + \min \left(2^{i+2},1 \right) \cdot (p_i^t - p) \leq p_i^t \text{.}    
\end{equation}
This means that the \textbf{contribution to $W$ increases additively by at most $p_i^t$}.

\paragraph{Type 3.} $i = -B_\e$: The contribution of $p_i^t$ to $W_t$ is equal to $2^{-B_\e} \cdot p_i^t$. By the rules of the process $p \leq p_i^t$ of the mass of $p_i^t$ moves to $i+2 = -B_\e +2$. Thus the contribution of this mass to $W_{t+1}$ is at most:
\begin{align}
&2^{-B_\e} \cdot (p_{-B_\e}^t - p) + 2^{-B_{\e}+2} \cdot p \nonumber \\
&\leq 2^{-B_{\e}+2} \cdot p_{-B_\e}^t \nonumber \\
&\leq 2\e^2 \nonumber \\
&\leq 2\e \text{,} \label{eq:leftendptbnd}
\end{align}
where the second to last inequality follows from the fact that $p_{-B_\e}^t \leq 1$ and that $B_\e = 2\ceil{\log(1/\e)}+1$. This means that the \textbf{contribution to $W$ increases additively by at most $2\e$}.

\paragraph{Type 4.} $i \in I_\e^+$: The contribution of $p_i^t$ to $W_t$ is equal to $p_i^t$. By the rules of the process $p \leq p_i^t$ of the mass moves to $i-2$ and $p_i^t - p$ of the mass moves to $i+2$ (or stays at $i$ if $i = B_\e$). Thus the contribution of this mass to $W_{t+1}$ is at most:
\begin{equation}\label{eq:positivecontribution}
\min \left(2^{i-2},1 \right) \cdot p + 1 \cdot (p_i^t - p) \leq p_i^t \text{.}    
\end{equation}
This means that the \textbf{contribution to $W$ decreases}.

Observe that the contribution to  $W$ can increase only for Type 2 and 3. Moreover, note that the total amount of mass that can be in Type 2 is at most $B_\e \cdot \frac{1}{B_\e^4} \sum_{j \in I_{\e}^+ } p_j^t \leq \frac{1}{B_\e^3} \sum_{j \in I_{\e}^+ } p_j^t$. Thus combining the observations for the 4 Types we get that:
$$
W_{t+1} 
\leq W_t + 2\e + \frac{1}{B_\e^3} \sum_{j \in I_{\e}^+ } p_j^t.
$$

Combined with the left-most inequality in \eqref{eq:wtuprbnd}, we get
\begin{equation}\label{eq:wtgeneralbound}
W_{t+1} \leq \left(1 + \frac{1}{B_\e^3} \right) W_t + 2\e . 
\end{equation}

Now we analyze the evolution of $W$ from time $t$ to time $t+1$ depending on how much mass moves from $-1$ to $+1$ and vice versa. Let $0 \leq \delta_{+,-} \leq p_1^t, 0 \leq \delta_{-,+} \leq p_{-1}^t$ be the amount of mass that is moved from $+1$ to $-1$ and from $-1$ to $+1$, respectively. We consider the following cases:

\paragraph{Case 1.} $\delta_{-,+} \geq \frac{1}{2 B_\e} \cdot \sum_{j \in I_{\e}^+ } p_j^t$: We will show that if $W_t \geq 16\e B_\e$, then $W_{t+1} \leq \left(1 - \frac{1}{9B_\e} \right)W_t$. 

First, observe that if $p_{-1}^t < \frac{1}{B_\e^4} \sum_{i \in I_\e^+} p_i^t$ then $p_{-1}^t < \frac{1}{2 B_\e} \cdot \sum_{j \in I_{\e}^+ } p_j^t \leq \delta_{-,+}$. This means that there is not enough mass on $-1$ to satisfy the assumption of this case. Thus we know that $p_{-1}^t \geq \frac{1}{B_\e^4} \sum_{i \in I_\e^+} p_i^t$. By the rules of the process only $\frac{1}{16}$ of $p_{-1}^t$ can potentially move to $+1$.
This implies that $p_{-1}^t \geq 16 \delta_{-,+} \geq \frac{16}{2 B_\e} \cdot \sum_{j \in I_{\e}^+ } p_j^t$. Using the properties of Type 1 we know that the contribution of $p_{-1}^t$ to $W$ goes down by at least a factor $2$.

Let us now look at the evolution of the contribution of $\sum_{i \in I_{\e}^+ } p_i^t$. We know from Type 4 that for this type the contribution does not increase. 

Since we know that $p_{-1}^t \geq \frac{8}{B_\e} \cdot \sum_{j \in I_{\e}^+ } p_j^t$, the total contribution of all the weight on $\{-1\} \cup  I_{\e}^+$ goes down by a factor of at least
\begin{align*}
   \frac{\frac{2}{B_\e}+1}{\frac{4}{B_\e}+1} \leq 1 - \frac{1}{B_\e} \text{,}
\end{align*}
where we used that $\e < \frac{1}{32}$.

It remains to include the contributions of the mass at $I_{\e}^- \setminus \{-1\}$. We know from the properties of Type 1, 2, and 3 that for those the contribution decreases by a factor $2$ with the exception of very "small" masses and the mass at the left-hand side boundary. Since $\frac{1}{2} \leq 1 - \frac{1}{B_\e}$ we get:


\begin{align}
W_{t+1} 
&\leq \left(1 - \frac{1}{B_\e} \right)\cdot W_t+  \frac{1}{ B_\e^3} \sum_{j \in I_{\e}^+ } p_j^t + 2\e   \nonumber \\
&\leq \left(1 - \frac{1}{B_\e} + \frac{1}{B_{\e}^3} \right)\cdot W_t+  2\e \text{,} 
\label{eq:case1bound}
\end{align}
where in the last inequality we used the left-most inequality from \eqref{eq:wtuprbnd}.

To get the claimed bound, observe that if $W_t \geq 16\e B_\e$, then by \eqref{eq:case1bound} we get that 
\begin{equation}\label{eq:case1final}
W_{t+1} \leq \left(1 - \frac{1}{9B_\e} \right)W_t \text{.}
\end{equation}

\paragraph{Case 2.} $\delta_{+,-} \geq \frac{1}{2 B_\e} \cdot \sum_{j \in I_{\e}^+ } p_j^t$: We will show that if $W_t \geq 16\e B_\e$, then $W_{t+1} \leq \left(1 - \frac{1}{9B_\e} \right)W_t$.  

First, observe that the contribution of $p_1^t$ to $W$ decreases by at least $\frac{1}{4 B_\e} \cdot \sum_{j \in I_{\e}^+ } p_j^t$. This is true since at least $\frac{1}{2 B_\e} \cdot \sum_{j \in I_{\e}^+ } p_j^t$ of the mass was moved from position $+1$ which is weighted by $1$ to position $-1$ which is weighted by $\frac12$. 

Let us now look at the evolution of the contribution of $I_{\e}^+ \setminus \{+1\}$. We know from Type 4 that for this type the contribution does not increase.

Since we know that the contribution of $p_1^t$ decreased by at least $\frac{1}{4 B_\e} \cdot \sum_{j \in I_{\e}^+ } p_j^t$ the total contribution of all the weight in $I_\e^+$ goes down by a factor of at least
$$1 - \frac{1}{4 B_\e} \text{.}$$

It remains to include the contributions of the mass at $I_{\e}^-$. We know from the properties of Type 1, 2, and 3 that for those the contribution decreases by a factor $2$ with the exception of very "small" masses and the mass at the left-hand side boundary. Since $\frac{1}{2} \leq 1 - \frac{1}{4B_\e}$ we get:

\begin{align}
W_{t+1} 
&\leq \left(1 - \frac{1}{4B_\e} \right) \cdot W_t +  \frac{1}{B_\e^3} \sum_{j \in I_{\e}^+} p_j^t + 2\e \nonumber \\ 
&\leq \left(1 - \frac{1}{4B_\e} + \frac{1}{B_{\e}^3} \right) \cdot W_t + 2\e \label{eq:case2bound}  
\end{align}
where in the last inequality we used the left-most inequality from \eqref{eq:wtuprbnd}.


To get the claimed bound, observe that if $W_t \geq 16\e B_\e$ then by \eqref{eq:case2bound} we get that 
\begin{equation}\label{eq:case2final}
W_{t+1} \leq \left(1 - \frac{1}{9B_\e} \right)W_t \text{.}
\end{equation}

\paragraph{Case 3.} $\delta_{-,+},\delta_{+,-} < \frac{1}{2 B_\e} \cdot \sum_{j \in I_{\e}^+ } p_j^t$: We will show that $M_{t+1} \leq M_t - 1$. 

For simplicity we introduce notation $\mu_t := M_t \cdot\sum_{j \in I_{\e}^+ } p_j^t$. First let's analyze what happens when $\delta_{-,+} = 0$. By the rules of the process at least $\frac{15}{16}\sum_{j \in I_{\e}^+ } j \cdot p_j^t$ of the mass on $I_{\e}^+ $ moves down. This and the assumption that $\delta_{+,-} < \frac{1}{2 B_\e} \cdot \sum_{j \in I_{\e}^+ } p_j^t$ gives the following two bounds: 
\begin{align*}
\mu_{t+1} 
&\leq \mu_t + \left(-2 \cdot \frac{15}{16} + 2 \cdot \frac{1}{16}\right) \cdot \sum_{j \in I_{\e}^+ } p_j^t + \frac{1}{2 B_\e} \cdot \sum_{j \in I_{\e}^+ } p_j^t \\
&\leq \mu_t + \left( \frac{1}{2 B_\e} - \frac{7}{4} \right) \cdot \sum_{j \in I_{\e}^+ } p_j^t
\end{align*}
and
\begin{align*}
\sum_{j \in I_{\e}^+ } p_j^{t+1} 
&\geq \sum_{j \in I_{\e}^+ } p_j^t - \frac{1}{2 B_\e} \cdot \sum_{j \in I_{\e}^+ } p_j^t \\
&= \left(1 - \frac{1}{2B_\e} \right) \sum_{j \in I_{\e}^+ } p_j^t
\end{align*}
Combining the two bounds we get: 
\begin{align}
&M_{t+1} = \frac{\mu_{t+1}}{\sum_{j \in I_{\e}^+ } p_j^{t+1}} \nonumber \\
&\leq \frac{\mu_t + \left( \frac{1}{2 B_\e} - \frac{7}{4} \right) \cdot \sum_{j \in I_{\e}^+ } p_j^t}{ \left(1 - \frac{1}{2B_{\e}} \right)\sum_{j \in I_{\e}^+ } p_j^t } \nonumber \\
&= \frac{2 B_\e}{2 B_\e - 1} \left(\frac{\mu_t}{\sum_{j \in I_{\e}^+ } p_j^t} - \frac74 \right) + \frac{1}{2B_\e - 1} \nonumber \\
&= \frac{2 B_\e}{2 B_\e - 1} \left(M_t - \frac74 \right) + \frac{1}{2B_\e - 1} \nonumber \\
&\leq M_t - 1 \label{eq:mtdecreases}
\end{align}
where in the last equality we used the definition of $\mu_t$. In the last inequality we used the fact that $M_t \leq B_\e$. Note that if $\delta_{-,+} \neq 0$ then $M_{t+1}$ can only decrease as $M_{t+1} \in [1,B_\e]$ and the mass that comes from $-1$ to $+1$ arrives at $1$. Thus by \eqref{eq:mtdecreases} we get that in Case 3:
\begin{equation}\label{eq:mtdecreasesbyconstant}
M_{t+1} \leq M_t - 1 \text{.}    
\end{equation}


\paragraph{Merging Cases 1, 2 and 3.} 
First, observe that Cases 1, 2 and 3 cover all potential values of $\delta_{-,+}$ and $\delta_{+,-}$. Then note that by \eqref{eq:wtgeneralbound} if $W_t \geq 10 \e B_\e^3$ then:
\begin{equation}\label{eq:wtgeneralboundforhighwt}
W_{t+1} \leq \left(1 +\frac{12}{10 B_\e^3} \right) W_t    
\end{equation}

Combining \eqref{eq:case1final}, \eqref{eq:case2final},  \eqref{eq:mtdecreasesbyconstant} and \eqref{eq:wtgeneralboundforhighwt} we get that if $W_t \geq \max (10\e B_\e^3, 16\e B_\e) = 10\e B_\e^3$ then either $W_t$ decreases by a multiplicative factor $\left(1 - \frac{1}{9B_\e} \right)$ or $M_t$ decreases by an additive $1$ and $W_t$ increases by at most a multiplicative factor of $\left(1 +\frac{12}{10 B_\e^3} \right)$. If on the other hand $W_t < 10\e B_\e^3$ then we have by the definition of $W$ that $W_{t+1} \leq 4 W_t$, as each amount of mass can increase it's contribution by at most a multiplicative factor of $4$.

As $M_t \in [1, B_\e]$ it means that in every consecutive $B_\e$ steps rule \eqref{eq:case1final}/\eqref{eq:case2final} is triggered at least once. Thus we have that:
\begin{align}
W_{t+ B_\e} 
&\leq \max \Bigg( \left(1 + \frac{12}{10 B_\e^3} \right)^{B_\e} \left(1 - \frac{1}{9 B_\e} \right) W_t, \nonumber \\
&40\e B_\e^3 \left(1 + \frac{12}{10 B_\e^3} \right)^{B_\e} \Bigg) \nonumber \\
&\leq \max \left( e^{\frac{12}{10B_\e^2}} \cdot e^{- \frac{1}{9 B_\e}} \cdot W_t, 64 \e B_\e^3 \right) \nonumber \\
&\leq \max \left( e^{-\frac{1}{81 B_\e}} \cdot W_t, 64 \e B_\e^3 \right) \label{eq:wtrecurrence} \text{.}
\end{align}
where in the first inequality we used that either for all $t \in [t,t+B_\e]$ we have that $W_t \geq 10\e B_\e^3$ or there exists $t' \in [t,t+B_\e]$ such that $W_{t'} < 10\e B_\e^3$. The two terms govern the first and the second case respectively. If $W_{t'} < 10 \e B_\e^3$ then $W$ grows by a multiplicative factor of at most $4$ in every step until it reaches at most $40\e B_\e^3$ and then it grows at most by a multiplicative factor of $1 + \frac{12}{10 B_\e^3}$ per step. In the last inequality we used that $\e < \frac{1}{32}$ and that $B_\e = 2\ceil{\log(1/\e)} +1$.

By the right-most inequality of \eqref{eq:wtuprbnd} we get that $W_0 \leq 1$, which by using \eqref{eq:wtrecurrence} implies that after $t = O(B_\e^3)$ steps $W_t \leq 64\e B_\e^3$, which by the left-most inequality of \eqref{eq:wtuprbnd} gives:
$$\sum_{i \in I_{\e}^+ } p_i^t \leq 64\e B_\e^3 \text{.}$$
\end{proof}

\thmqc*

\begin{proof}
Observe that the number of queries $q$ asked by the algorithm is upper bounded by:
\begin{align*}
&q \leq t \cdot (B_{\e'}+1) \cdot m \\
&\leq O \left(B_{\e'}^3 \cdot B_{\e'} \cdot (d + \log(B_{\e'}^4) + \log(1/\delta)) \cdot B_{\e'}^4 \right) \\
&\leq O \left((d + \log(B_{\e'}^4) + \log(1/\delta))B_{\e'}^8  \right)\text{.}
\end{align*}
Note that 
$$B_{\e'} \leq O\left(\log \left(\frac{\log(1/\e)}{\e}\right) \right) \leq O(\log(1/\e)) \text{.}$$
Thus combining the two bounds we get an upper bound for the number of queries:
\begin{align*}
q 
&\leq O((d + \log\log(1/\e) + \log(1/\delta))\log^8(1/\e)) \\
&\leq O((d + \log(1/\delta))\log^9(1/\e)) \text{.}
\end{align*}

Now we prove the correctness of the algorithm. First observe that by the definition of $m$ and the union bound over $O(B_{\e'}^4)$ many events we know that success events of Lemmas~\ref{lem:34ofpositivegoesdown} and \ref{lem:ifbiggerthanthreshgoesdown} hold when lemmas are applied to functions of the form $\text{Maj}(h_1,\dots, h_{i-1}), \text{Vote}_g(h_1, \dots, h_{i-1})$, where $h_1, \dots, h_t$ are functions constructed throughout the algorithm. Observe then that if for every $i \in I_{\e'}, j \in [t]$ $\left(t = O \left(B_{\e'}^3 \right)\right)$ we define:
$$p_i^j := \mathbb{P}_{x \sim \mathcal{D}}[\text{Vote}_g(h_1, \dots, h_j)(x) = i] \text{,}$$
then $\{ \{p_i^j\}_{i \in I_{\e'}} \}_{j \in [t]}$ satisfies the rules of the process on $2\Z + 1$ with parameter $\e'$ (Definition~\ref{def:process}). Lemma~\ref{lem:34ofpositivegoesdown} is responsible for the first property and Lemma~\ref{lem:ifbiggerthanthreshgoesdown} is responsible for the second property. Thus we can apply Lemma~\ref{lem:processconverges} to $\{ \{p_i^j\}_{i \in I_{\e'}} \}_{j \in [t]}$ to get that at the end of the process we have:
$$
\sum_{i \in I_{\e'}^+ } p_i^t 
\leq 64 \cdot \e' \cdot B_{\e'}^3
\leq \e \text{,}   
$$ 
where in the last inequality we used the assumption that $\e < \frac{1}{32}$. To conclude observe that:
\begin{align*}
\sum_{i \in I_{\e'}^+ } p_i^t 
&= \sum_{i \in I_{\e'}^+ } \mathbb{P}_{x \sim \mathcal{D}}[\text{Vote}_g(h_1, \dots, h_t)(x) = i] \\
&= R(\text{Maj}(h_1, \dots, h_t)) \text{.}
\end{align*}

\end{proof}


\begin{proof}[Proof of Theorem~\ref{thm:ip}.]
Let $X$ be a feature space, $\e \in (0,\frac{1}{32})$, $\mathcal{H}$ be a hypothesis class of VC-dimension $d$. We will show that EQ-learner (Algorithm~\ref{alg:learner}) satisfies the conditions of the theorem.

Assume that the EQ-learner is run with parameters $\e, \delta = 1/3, \mathcal{H}$ and every call to $\text{EQ}_{\mathcal{D}}$ replaced by an interaction with \textbf{A}. This setup satisfies the requirements of the adversarial learning game (Definition~\ref{def:game}). Now there are two possible scenarios. First scenario is that throughout the run of the algorithm, for all functions $f_t$ that \textbf{L} presents to \textbf{A} we have that $\mathbf{A}(f_t, EX_{\mathcal{D}}) = EQ_{\mathcal{D}}(f_t)$. Then Theorem~\ref{thm:querycomplexity} guarantees that with probability $2/3$ the first statement of the theorem is true. The other scenario is that there exists $f_t$ that \textbf{L} presented to \textbf{A} such that $\mathbf{A}(f_t, EX_{\mathcal{D}}) \neq EQ_{\mathcal{D}}(f_t)$. This implies the second statement of the theorem.

What is left is to observe that the EQ-learner queries the $\text{EQ}_{\mathcal{D}}$ only for $O(\text{polylog}(1/\e))$ many different functions. This is true as the number of different functions sent to $\text{EQ}_{\mathcal{D}}$ is upper-bounded by $t \cdot (B_{\e'} + 1) \leq O(\text{polylog}(1/\e))$, where parameters $t$ and $B_{\e'}$ are defined in the algorithm.
\end{proof}

\end{document}